\documentclass[letterpaper, 10 pt, conference]{ieeeconf} 

\IEEEoverridecommandlockouts 
\usepackage{graphicx}
\usepackage{amsmath}
\usepackage{arydshln}
\usepackage{hhline}
\usepackage{eucal}
\usepackage{cite}
\usepackage{amssymb}
\usepackage{enumerate}
\usepackage{subfigure}
\usepackage{bm}
\usepackage{color}
\usepackage{multirow}
\newtheorem{thm}{Theorem}
\newtheorem{lem}{Lemma}

\newtheorem{rem}{Remark}

\newtheorem{assumption}{Assumption}

\newcommand*{\QEDB}{\hfill\ensuremath{\square}}%

\newcommand{\bJ}{\bm{J}}

\newcommand{\bT}{\bm{T}}
\newcommand{\bM}{\bm{M}}
\newcommand{\bC}{\bm{C}}
\newcommand{\bD}{\bm{D}}
\newcommand{\bK}{\bm{K}}

\newcommand{\bg}{\bm{g}}

\newcommand{\bxi}{\bm{\xi}}
\newcommand{\bbxi}{\bar{\bm{\xi}}}

\newcommand{\bEta}{\bm{\eta}}
\newcommand{\bbM}{\bar{\bm{M}}}
\newcommand{\bbC}{\bar{\bm{C}}}

\newcommand{\bq}{\bm{q}}

\newcommand{\br}{\bm{r}}
\newcommand{\bv}{\bm{v}}
\newcommand{\bV}{\bm{V}}

\newcommand{\bw}{\bm{w}}

\newcommand{\bu}{\bm{u}}
\newcommand{\bx}{\bm{x}}
\newcommand{\by}{\bm{y}}
\newcommand{\bphi}{\bm{\phi}}
\newcommand{\btau}{\bm{\tau}}
\newcommand{\bzero}{\bm{0}}
\newcommand{\bI}{\bm{I}}
\newcommand{\bR}{\bm{R}}
\newcommand{\bQ}{\bm{Q}}
\newcommand{\bF}{\bm{F}}

\newcommand{\bLambda}{\bm{\Lambda}}
\newcommand{\bGamma}{\bm{\Gamma}}

\title{\LARGE \bf
Passive Compliance Control of Aerial Manipulators
}

\author{Min Jun Kim, Ribin Balachandran, Marco De Stefano, Konstantin Kondak, and Christian Ott
\thanks{The funding of the European Commission to the AEROARMS project under the H2020 Programme (Grant Agreement 644271) is acknowledged.}
\thanks{The authors are with Institute of Robotics and Mechatronics, German
Aerospace Center (DLR),  Wessling, Germany. E-mail: minjun.kim@dlr.de}%
}

\begin{document}

\maketitle
\thispagestyle{empty}
\pagestyle{empty}

\begin{abstract}
This paper presents a passive compliance control for aerial manipulators to achieve stable environmental interactions. The main challenge is the absence of actuation along body-planar directions of the aerial vehicle which might be required during the interaction to preserve passivity. The controller proposed in this paper guarantees passivity of the manipulator through a proper choice of end-effector coordinates, and that of vehicle fuselage is guaranteed by exploiting time domain passivity technique. Simulation studies validate the proposed approach.
\end{abstract}

\section{Introduction}

After successful achievements in unmanned aerial vehicle (UAV) studies for non-active missions such as surveillance and remote sensing, aerial manipulation is an emerging research topic. By mounting a manipulator (or multiple manipulators) to a UAV, we can exploit further capabilities of aerial platforms. UAV equipped with a manipulator (UAV-M hereinafter), for instance, allows us to extend the missions to active ones such as grasping and manipulation which may include interaction with environments.

A number of studies have been performed to achieve successful aerial manipulation in various perspectives; e.g., UAV-M design \cite{orsag2013modeling,suarez2016lightweight, ryll20176d, kim2018oscillation}, intelligence \cite{karrer2016real, rossi2017trajectory, pumarola2017pl}, and modeling methodologies \cite{yang2014dynamics,garofalo2015inertially}. In addition to these, control of UAV-M is also extensively studied. For example, \cite{kim2013aerial, huber2013first, kondak2014aerial, yang2014dynamics, tognon2017dynamic,gianluca2018taskspace} studied stability of UAV-M systems without considering manipulation tasks explicitly. \cite{mellinger2011design, fink2011planning, jimenez2013control, orsag2013stability} tackled actual aerial manipulation tasks, but the compliance interaction control was not the main scope of these studies mainly because the UAVs were not equipped with multi degrees of freedom (DoF) robotic manipulators. In contrast, the system considered in this work is equipped with a multi-DoF robotic manipulator as shown in Fig. \ref{fig:heli_manipulator} (see \cite{huber2013first} for more details). To achieve further capabilities of the multi-DOF manipulators, \cite{lippiello2012cartesian} and \cite{lippiello2012exploiting} formulated Cartesian impedance control approaches for UAV-M systems.


\begin{figure}
	\centering
	\subfigure[]
	{\includegraphics[width=76mm]{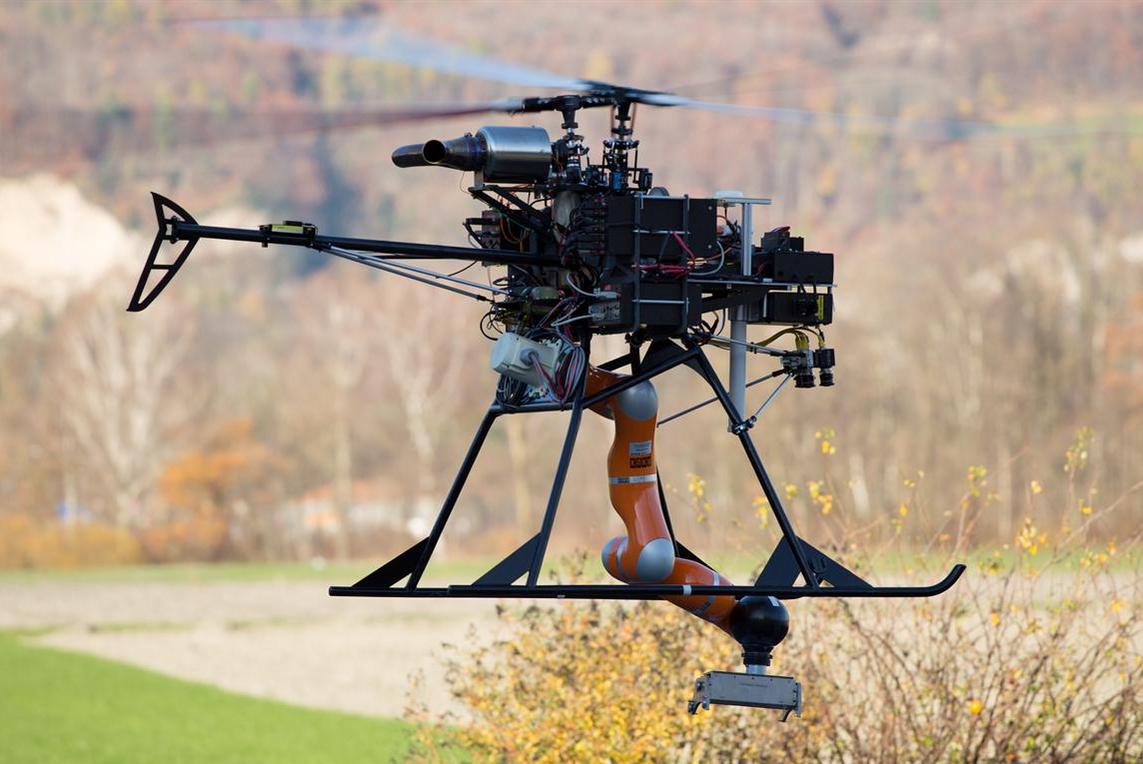}} \\
	\centering
	\subfigure[]
	{\includegraphics[width=80mm]{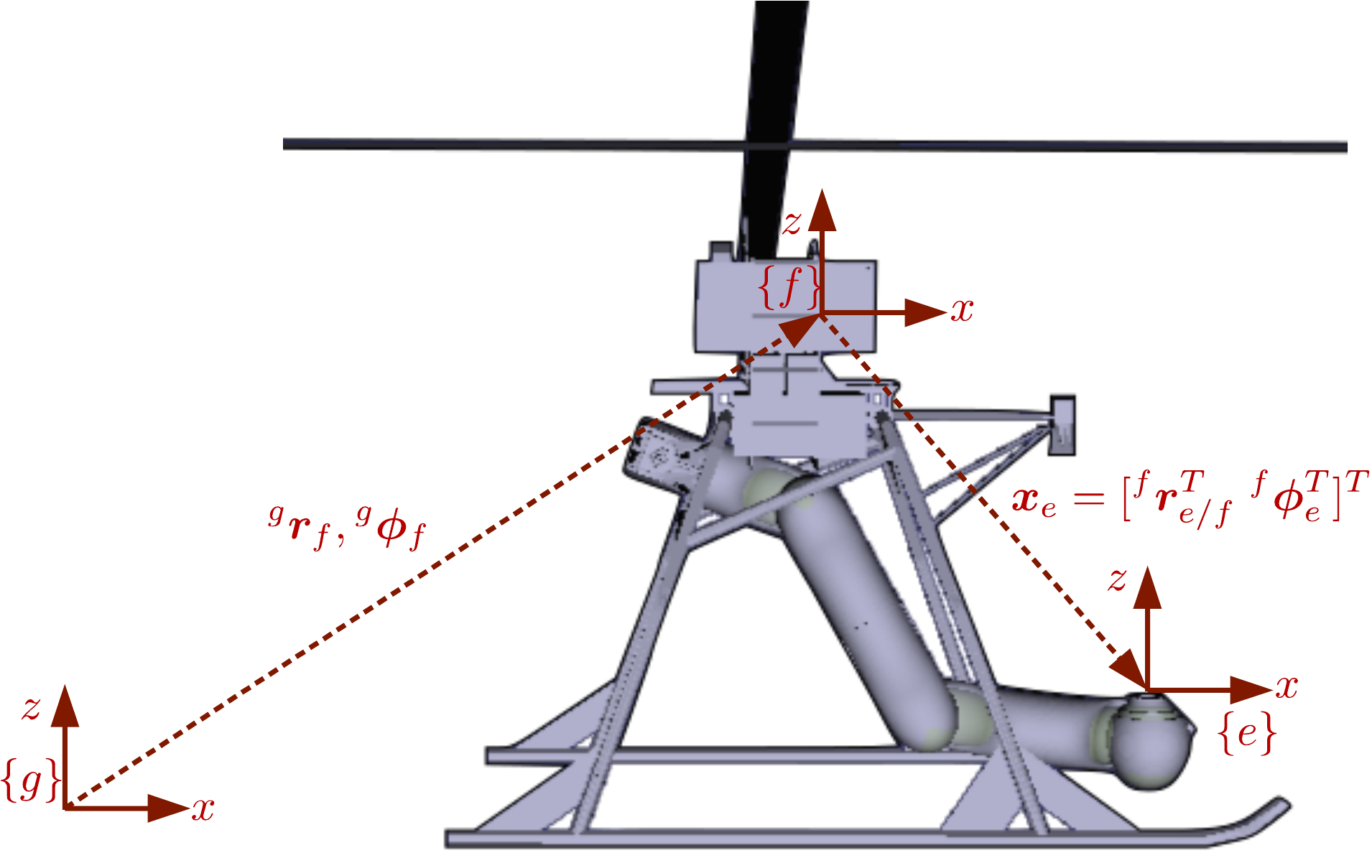}}
	\caption{(a) UAV-M system developed in DLR. (b) Coordinate systems of UAV-M. Note that the system does not have actuation along the body $x,y$ directions (under-actuation).}
\label{fig:heli_manipulator}	
\end{figure}

In this paper, we extend our previous work \cite{kim2018stabilizing} in which stable flight control for UAV-M was studied, however the interaction with the environment was not considered. In particular, we extend the formulation from using joint level control to using end-effector control in order to achieve compliance behavior in the task space. Realizing stable environmental interaction, however, is not trivial when the UAV is under-actuated because the UAV lacks actuation along body linear $x$-, $y$-directions (see Fig. \ref{fig:heli_manipulator}b).\footnote{In fact, fully actuated platforms are recently being studied \cite{rajappa2015modeling,8281444}.}  During the interaction, the interaction force (or torque) propagates through the manipulator to the fuselage. As a result, forces along the linear $x$-,$y$-directions are generated. These forces can not be directly handled due to the lack of actuation and a stable interaction with the environment is not guaranteed.


The contribution of this paper is to design a passive compliance control for under-actuated UAV-M systems to ensure stable interactions with the passive environments. Passivity of UAV's $z$-directional translation dynamics and UAV's rotational dynamics is rather straightforward by virtue of the collocated actuations. This paper exploits time domain passivity approach \cite{hannaford2002time,ryu2004stable} in which passivity observer/passivity controller (PO/PC) is used to render the system passive. However, passivity of the compliance controller of the manipulator, which may result in forces along non-actuated directions, is not trivial. In this paper, it will be shown that the passivity is guaranteed if the compliance controller is formulated in the UAV fuselage frame (i.e., using $\bx_e$ in Fig. \ref{fig:heli_manipulator}b). As a result, stability is guaranteed when the controlled UAV-M interacts with passive environments.

The resulting controller has three additional benefits. First, although the UAV dynamics and the manipulator dynamics are highly coupled, they are controlled independently except for the gravity compensation. This may significantly reduce the effort and time for implementation. Second, the resulting control law is almost model-free; the total mass and gravity vector are the only required modeling parameters. Third, the end-effector compliance control is formulated using $\bx_e$ which is expressed in the fuselage frame. For example, this can be used for visual servoing control, where the desired values are usually expressed in the same frame $\{f\}$.


The rest of the paper is organized as follows. Section \ref{sec:modeling} presents the mathematical modeling of the UAV-M. Section \ref{sec:control} presents the passive compliance control of the manipulator and UAV fuselage control with simulation validation in Section \ref{sec:validation}. Section \ref{sec:conclusion} concludes the paper.

\section{Modeling of UAV-Manipulator Systems}
\label{sec:modeling}

Considering the following assumption, this section introduces mathematical modeling of the UAV-M.
\begin{assumption}
	\label{ass:redundancy}	
	This paper considers  non-redundant robot manipulators and assumes that the manipulator is not in singular configurations. In particular, this paper will consider 6-DOF manipulator to exploit full task space. \QEDB
\end{assumption}

In Fig. \ref{fig:heli_manipulator}b, $\{g\}$, $\{f\}$, and $\{e\}$ represent global, UAV fuselage and end-effector frames respectively. The frame $\{f\}$ is located at the center of mass (CoM) of the UAV, not UAV-M. In addition, UAV-M velocity vectors are defined by
\begin{align}
\label{eq:velocities}
&\bxi_q = \left(
\begin{array}{c}
\bV_f \\
\hdashline 	\dot{\bq}_m
\end{array}
\right), \;\;
\bxi_{e} = \left(
\begin{array}{c}
\bV_f \\
\hdashline \bV_e
\end{array}
\right),
\end{align}
where $\bV_{(\cdot)}$ denotes the body twist of the frame $\{(\cdot)\}$. $\bq_m$ is the generalized coordinates of the manipulator. After introducing a commonly used UAV-M dynamics expressed in $\bxi_q$, the coordinate will be transformed into $\bxi_e$.


Equation of motion of the UAV-M in the $\bxi_q$ coordinate is given by
\begin{align}
\label{eq:modified_dyn}
\bM(\bq_m) \dot{\bxi}_q + \bC(\dot{\bq}_m,\bxi_q) \bxi_q + \bg({}^{g}\bphi_f, \bq_m) = \btau_b + \bJ_{e}^{T} {}^{e}\bEta_{\text{ext}}
\end{align}
with inertia matrix $\bM$, Coriolis/centrifugal matrix $\bC$, and gravity vector $\bg$.  ${}^{g}\bphi_f$ is the RPY angle associated with the rotation matrix ${}^{g}\bR_{f}$  which represents the rotation from $\{g\}$ to $\{f\}$. $\btau_b \in \Re^{12}$ represents the control command in the body frame, and is given by
\begin{align}
\btau_b =  [0 \; 0 \; f_{\text{th}} \; \btau_{\text{uav}}^T \; \btau_m^T]^T,
\end{align}
where $f_{\text{th}} \in \Re$ is UAV thrust, $\btau_{\text{uav}} \in \Re^{3}$ is the torque around the CoM of UAV, and $\btau_m \in \Re^{6}$ is the joint torque of the manipulator. Note that, by expressing the dynamic model in the body frame, the first two elements of $\btau_b$ are zero. UAV-M external force/moment vector ${}^{e}\bEta_{\text{ext}} \in \Re^{12}$  is given by
\begin{align}
{}^{e}\bEta_{\text{ext}} = \left(
\begin{array}{c}
\bzero_{6\times 1} \\
{}^{e}\bF_{\text{ext}}
\end{array}
\right),
\end{align}
where ${}^{e}\bF_{\text{ext}} \in \Re^{6}$ represents the body wrench due to the environmental interaction.

The Jacobian matrix $\bJ_e \in \Re^{12 \times 12}$ defines the relation between the velocity vector $\bxi_{e}$ and $\bxi_q$ by
\begin{align}
\bxi_{e} = 
\underbrace{
	\left[
	\begin{array}{ccc}
	\bI &  \bzero \\ 
	\bJ_{ef} & 	\bJ_{eq}
	\end{array}
	\right]
}_{=\bJ_e}
\bxi_q.
\end{align}
Using this, the UAV-M dynamics (\ref{eq:modified_dyn}) can be rewritten in $\bxi_e$ coordinates as
\begin{align}
\label{eq:modeling_e}
\bLambda \dot{\bxi}_e + \bGamma \bxi_e + \bJ_e^{-T} \bg = \bJ_e^{-T}\btau_b + {}^{e}\bEta_{\text{ext}},
\end{align}
where $\bLambda$ and $\bGamma$ are the inertia and Coriolis/centrifugal matrices in the new coordinate system. Moreover, using
\begin{align}
\label{eq:J_e_inv}
\bJ_e^{-T}= 
\left[
\begin{array}{ccc}
 \bI & -\bJ_{ef}^T \bJ_{eq}^{-T} \\ 
 \bzero & 	\bJ_{eq}^{-T}
\end{array}
\right],
\end{align}
which is valid due to Assumption \ref{ass:redundancy}, (\ref{eq:modeling_e}) can be rewritten as
\begin{align}
\label{eq:modeling_e_re}
\bLambda \dot{\bxi}_e + \bGamma \bxi_e + \bJ_e^{-T} \bg = {}^{e}\bEta_{\text{th}} + {}^{e}\bEta_{\text{uav}} + {}^{e}\bEta_{m} + {}^{e}\bEta_{\text{ext}},
\end{align}
where ${}^{e}\bEta_{\text{th}}, {}^{e}\bEta_{\text{uav}}, {}^{e}\bEta_{m} \in \Re^{12}$ are defined by
\begin{align}
\label{eq:e_eta_uav}
&{}^{e}\bEta_{\text{th}} = 
\left(
\begin{array}{c}
\bzero_{2 \times 1} \\
f_{\text{th}} \\
\bzero_{9 \times 1}
\end{array}
\right)
, \;
{}^{e}\bEta_{\text{uav}} = 
\left(
\begin{array}{c}
\bzero_{3 \times 1} \\
\btau_{\text{uav}} \\
\bzero_{6 \times 1}
\end{array}
\right)
, \\
\label{eq:e_eta_m}
& \qquad \qquad 
{}^{e}\bEta_{m} = 
\left[
\begin{array}{c}
-\bJ_{ef}^T \bJ_{eq}^{-T} \\
\bJ_{eq}^{-T}
\end{array}
\right]
\btau_m.
\end{align}
Note that $f_{\text{th}}$ influences body $z$-direction dynamics, and $\btau_{\text{uav}}$ influences UAV rotational dynamics. However, $\btau_m$ influences the whole UAV-M dynamics.

Later, the end-effector compliance control will be formulated using $\bx_e \in \Re^{6}$ defined by
\begin{align}
\label{eq:x_e_def}
\bx_{e} :=
\left(
\begin{array}{c}
{}^{f}\br_{e/f} \\
{}^{f}\bphi_{e}
\end{array}
\right), \;\;
\dot{\bx}_{e} :=
\left(
\begin{array}{c}
{}^{f}\dot{\br}_{e/f} \\
{}^{f}\dot{\bphi}_{e}
\end{array}
\right),
\end{align}
where ${}^{f}\bphi_{e}$ is the RPY angle associated with ${}^{f}\bR_e$, and ${}^{f}\br_{e/f}$ is the displacement vector of $\{e\}$ with respect to $\{f\}$, represented in $\{f\}$.

For future convenience, notations used throughout the paper are summarized as follows.
\begin{itemize}
	\item $\alpha$, $\beta$, $\gamma$: RPY (roll-pitch-yaw) angles.
		\begin{itemize}
			\item ${}^{g}\bphi_f=[\alpha_f \;\; \beta_f \;\; \gamma_f]^T$.
			\item ${}^{f}\bphi_e=[\alpha_e \;\; \beta_e \;\; \gamma_e]^T$.
		\end{itemize}
	\item ${}^{a}\bm{(\cdot)}_{b/c}$: vector $\bm{(\cdot)}$ pointing $\{b\}$ with respect to $\{c\}$, represented in $\{a\}$.
	\item ${}^{g}\br_{(\cdot)}={}^{g}\br_{(\cdot)/g}$: Simplification for position in $\{g\}$.
		\begin{itemize}
			\item ${}^{g}\br_{f}={}^{g}\br_{f/g}$ and ${}^{g}\br_{e}={}^{g}\br_{e/g}$
		\end{itemize}	
	\item Frames are omitted for body velocities.
		\begin{itemize}
			\item $\bv_f = {}^{f}\bv_{f/g}$ and $\bw_f = {}^{f}\bw_{f/g}$.
			\item Recall body twists $\bV_f$ and $\bV_e$ in (\ref{eq:velocities}).
		\end{itemize}
	\item ${}^{a}\bV_{b/c}=[{}^{a}\bv_{b/c}^T\; {}^{a}\bw_{b/c}^T]^T$.
	\item $\bQ$: Maps Euler angle rate to the angular velocity.
		\begin{itemize}
			\item $\bw_{f} = {}^{f}\bQ_{g} {}^{g}\dot{\bphi}_{f}$.
			\item $\bw_{e} = {}^{e}\bQ_{f} {}^{f}\dot{\bphi}_{e}$.
		\end{itemize}
	\item Vector components: $\bv_{f}=[v_{f,x}\; v_{f,y}\; v_{f,z}]^T$. This rule also applies to the other vectors.
	\item  Inertia matrix can be partitioned as follows.
	\begin{align}
	\bM|_{\text{first $3 \times 12$}} = & \left[ \bM_{tt} \; \bM_{tr} \; \bM_{tm} \right]
	\end{align}
	Here, the subscripts $t$, $r$, and $m$ represent `translational', `rotational', and `manipulator', respectively. Similarly, $\bg_t$ represents the first three elements of $\bg$.
\end{itemize}

Finally, the following assumptions on passivity of environment and UAV-M are made.
\begin{assumption}
	\label{ass:passive_env}
	Manipulator may interact with the environment of which the input-output (I/O) pair $(-\bV_e, {}^{e}\bF_{\text{ext}})$ is strictly passive.  \QEDB
\end{assumption} 
\begin{assumption}
	\label{ass:strictly_passive_UAV_M}
	The dynamics of UAV-M is subjected to drag although it is omitted in (\ref{eq:modeling_e_re}) for simplicity. Therefore, the UAV-M dynamics is output strictly passive. \QEDB
\end{assumption}

\section{Passive Compliance Control of UAV-M}
\label{sec:control}

\subsection{Control goal}

This paper tackles a compliance control of aerial manipulator with the following particular scenario. (i) The UAV-M system approaches a target position (free-flight), (ii) and the manipulator end-effector interacts with a passive environment while the UAV is keeping its desired position. As stated earlier, this problem is not trivial as the typical UAV lacks actuations along body $x$- and $y$-directions.  As shown in Fig. \ref{fig:problem_state}a, the forces acting on the end-effector will propagate through the UAV-M body, and will eventually result in body $x$- and $y$-directional forces in the UAV fuselage. Since there is no actuation to counteract these forces, stability of the resulting closed-loop  dynamics cannot be guaranteed. 


At this point, it is interesting to note that, in (\ref{eq:modeling_e_re})-(\ref{eq:e_eta_m}), the UAV control inputs (${}^{e}\bEta_{\text{th}}$, ${}^{e}\bEta_{\text{uav}}$) do not directly influence the manipulator dynamics. In contrast, the manipulator control input (${}^{e}\bEta_{m}$) influences the UAV dynamics directly. Keeping these in mind, this paper takes the following control strategy. 
\begin{itemize}
	\item UAV control inputs ${}^{e}\bEta_{\text{th}}$ and ${}^{e}\bEta_{\text{uav}}$ only take care of stabilization of UAV dynamics. Regardless of the manipulator dynamics (which may have environmental interaction), the UAV tries to maintain its desired position. Time domain passivity approach (in particular, PO/PC) will be applied to passivate the UAV controller.
	\item Manipulator control ${}^{e}\bEta_{m}$ is designed to render PD compliance behavior at the end-effector. However, because this control input applies forces along UAV's body $x$, $y$ directions which cannot be handled, ${}^{e}\bEta_{m}$ will be designed to be intrinsically passive (i.e., passive without passivation technique such as PO/PC). Later, it will be shown that the compliance controller is intrinsically passive if it is formulated using $\bx_e$.
\end{itemize}

\begin{figure}
	\centering
	\subfigure[]
	{\includegraphics[scale=0.525]{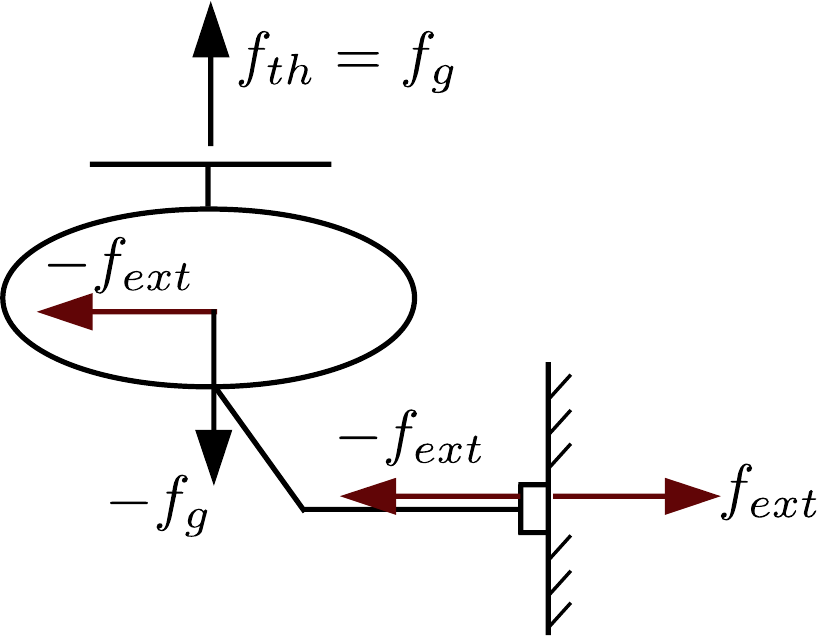}}
	\centering
	\subfigure[]
	{\includegraphics[scale=0.525]{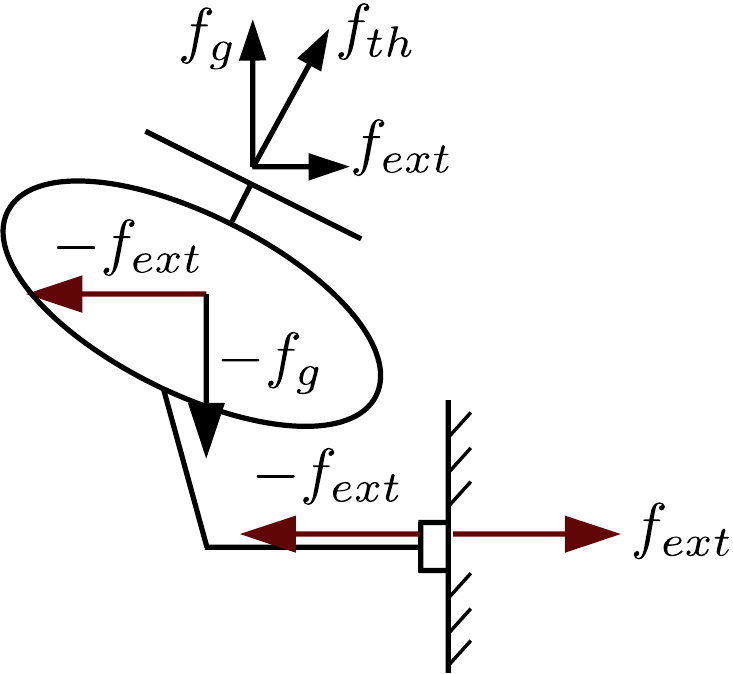}}
	\caption{(a) Due to the lack of actuation in UAV body $x$-direction, environmental interaction may generate energy along this direction. (b) By virtue of passivity, an open I/O port of the controlled UAV-M dynamics can be closed by the passive environment while maintaining stability. As a result, the UAV-M converges to balancing configuration. }
	\label{fig:problem_state}
\end{figure}

More specifically, UAV-M controller should fulfill the following control goals.
\begin{itemize}
	\item Asymptotic stability for the free-flight.\footnote{Based on our scenario of interest, UAV positioning is a global mission to approach the target (using, e.g., GPS), and manipulation is a local mission (using, e.g., vision system on the UAV fuselage).}
	\begin{itemize}
		\item The desired values of UAV are given in the global frame: ${}^{g}\br_f^{des}$, ${}^{g}\bphi_f^{des}=[0 \; 0 \; \gamma_f^{des}]^T$, with zero derivatives. Note that desired roll and pitch angles are zero because, otherwise, the UAV will move.\footnote{
			Unlike most of UAV controllers that update desired orientation in the control loop to achieve zero translational error \cite{ryll20176d, yang2014dynamics, gianluca2018taskspace}, the desired roll and pitch angles of the proposed controller  are always zero. Non-zero translational error perturbs the stable rotational dynamics, which tries to achieve zero roll and pitch angles, in such a way that the error is decreased.
		}
		\item The desired value of the end-effector is given in the fuselage frame: $\bx_e^{des}$ with zero derivative.	
	\end{itemize}
	\item Stable interaction with passive environment. 
	\begin{itemize}
		\item The UAV control should satisfy passivity of I/O pairs $(-\bxi_e, {}^{e}\bEta_{\text{th}})$ and $(-\bxi_e, {}^{e}\bEta_{\text{uav}})$.
		\item The compliance control should satisfy passivity of the I/O pair $(-\bxi_e, {}^{e}\bEta_{m})$.
	\end{itemize}
\end{itemize}
It should be remarked that stable interaction in this paper indicates asymptotic stability of the  entire control loop including the passive environment; see also Fig. \ref{fig:overall_control_str}. Therefore, the UAV-M will converge to an equilibrium point at which it can balance the interaction forces, as shown in Fig. \ref{fig:problem_state}b. We would like to emphasize that  the UAV-M automatically converges to this equilibrium point without calculating it.



\subsection{Control design}
\label{sec:control_design}

This section presents a passive and stable compliance controller for UAV-M, as an extension of our previous study \cite{kim2018stabilizing}. The proposed controller in this paper differs from the one in \cite{kim2018stabilizing} in the following aspects:
\begin{itemize}
	\item Formulation is extended from joint space tracking control to end-effector compliance control.
	\item UAV and manipulator are controlled independently.
	\item The main interest of \cite{kim2018stabilizing} was stability in free-flight. This paper additionally tackles stable environmental interaction.
	\item The position controller of the UAV (fuselage) includes an integral-term in order to maintain its position against interaction forces.
\end{itemize}
In the following, the UAV controller and manipulator controller will be presented.

\subsubsection{UAV control}

The UAV control law is given by:
\begin{align}
\label{eq:control_f_th}
f_{\text{th}} &= -D_z \tilde{v}_{f,z} - D^{pc}_z v_{f,z} +g_{t,z},\\
\label{eq:control_tau_uav}
\btau_{\text{uav}} &=  -\bD_w (\bw_f - \bw_f^{ref}) - \bD_w^{pc}\bw_f + \bg_r,
\end{align}
where $D_z>0$, $\bD_w$ are control gains for UAV control. $D^{pc}_zv_{f,z}$ and $\bD_w^{pc}\bw_f$, which will be defined shortly, are activated to passivate the UAV dynamics only when the passivity condition is violated \cite{hannaford2002time}.

In (\ref{eq:control_tau_uav}), $\bw_f^{ref}$ is obtained by integrating the reference acceleration $\dot{\bw}_f^{ref}$ defined by
\begin{align}
\label{eq:bw_f_dot_ref}
\dot{\bw}_f^{ref} =& -\bD (\bD_{w} \tilde{\bw}_f - \bM_{tr}^T \tilde{\bv}_f)  + \frac{d}{dt}({}^{f}\bQ_{g} \bD_\phi {}^{g}\bphi_f),
\end{align}
with
\begin{align}				
\label{eq:tilde_bw_original}
\tilde{\bw}_f =& \bw_f +{}^{f}\bQ_{g} \bD_{\phi} ({}^{g}\bphi_f - {}^{g}\bphi_f^{des}), \\
\nonumber\tilde{\bv}_f =& \bv_f + {}^{b}\bR_g \bD_{r,p} ({}^{g}\br_f - {}^{g}\br_f^{des})\\
              &+ {}^{b}\bR_g \bD_{r,i} \int ({}^{g}\br_f - {}^{g}\br_f^{des}).
\label{eq:tilde_bv_original}
\end{align}	
Here, $\bD, \bD_{r,p}, \bD_{r,i}, \bD_{\phi}>0$ are diagonal gain matrices. Integral action is included in (\ref{eq:tilde_bv_original}) to hold UAV's desired position against interaction forces. Since the integrator is a non-passive element, it may result in energy generation. In this paper, PO/PC techniques will be exploited to ensure passivity, and therefore, non-passive actions (e.g., integral) will be corrected if needed.


The UAV position error perturbs the rotational dynamics via $\bM_{tr}^T\tilde{\bv}_f$ in (\ref{eq:bw_f_dot_ref}).\footnote{Note that $\bM_{tr}$ is given by $\bM_{tr}=-m{}^{f}\br_{CoM/f}^{\vee}$. Here, $m$ is the total mass, ${(\cdot)}^{\vee}$ is the skew-symmetric operator, and ${}^{f}\br_{CoM/f}$ is the position of the CoM of the overall UAV-M system from the origin of $\{f\}$.} Because the UAV $x$, $y$ positions are indirectly controlled by orientation of the UAV, gains for rotational dynamics ($\bD_{w}$, $\bD_{\phi}$) should be selected larger than those for translational dynamics ($\bD_{r,p}$, $\bD_{r,i}$). Otherwise, the perturbation will be too large, and may result in insufficient performance. For more details on the gain selection strategy, one may refer to Section IV-A in \cite{kim2018stabilizing}. 

Note that, in the UAV controller (\ref{eq:control_f_th})-(\ref{eq:control_tau_uav}), manipulator variable does not appear explicitly,  except for the gravity compensation (i.e., controlled independently). Conceptually, $f_{\text{th}}$ leads to $\tilde{v}_{f,z}=0$ and $\btau_{\text{uav}}$ to $\dot{\bw}_f = \dot{\bw}_f^{ref}$. $\bv_{f,xy}$ is indirectly controlled by the reference acceleration tracking of $\dot{\bw}_{f}^{ref}$ which is designed to achieve $\tilde{\bw}_f=\bzero$ and $\tilde{\bv}_f=\bzero$ that imply stable error dynamics:
\begin{align}
{}^{g}\dot{\br}_f  + \bD_{r,p} ({}^{g}\br_f - {}^{g}\br_f^{des}) + \bD_{r,i} \int ({}^{g}\br_f - {}^{g}\br_f^{des})=& \bzero,\;\;\; \\ 
{}^{g}\dot{\bphi}_f  + \bD_\phi ({}^{g}\bphi_f - {}^{g}\bphi_f^{des}) = \bzero&,
\label{eq:meaning_tilde_zero}
\end{align}
because $\bw_f={}^{f}\bQ_{g} {}^{g}\dot{\bphi}_f$ and $\bv_f = {}^{f}\bR_g {}^{g}\dot{\br}_f$.

On the other hand, from the energy point of view, passivity of (\ref{eq:control_f_th})-(\ref{eq:control_tau_uav}) is not guaranteed. To overcome this, the following PO is applied to check if the passivity is violated:
\begin{align}
E_{obs,z} = \int -{}^{e}\bEta_{\text{th}}^T \bxi_{e} \mathrm{d}t = \int -f_{\text{th}} v_{f,z} \mathrm{d}t, \\
E_{obs,w} = \int -{}^{e}\bEta_{\text{uav}}^T \bxi_{e} \mathrm{d}t = \int -\btau_{\text{uav}}^T\bw_f \mathrm{d}t.
\end{align}
If the passivity condition is violated, the time varying damping terms $D^{pc}_z$ and $\bD_w^{pc}$ are applied to guarantee the passivity of UAV dynamics:
\begin{align}
D_z^{pc}&=\begin{cases}
-\frac{E_{obs,z} + E_{z}(0)}{dT \cdot v_{f,z}^2}, & \text{$E_{obs,z} < - E_{z}(0)$}\\
0, & \text{otherwise},
\end{cases} \\
\bD_w^{pc}&=\begin{cases}
-\frac{E_{obs,w}+E_{w}(0)}{dT \cdot \bw_{f}^T\bw_{f}}\bI, & \text{$E_{obs,w} < - E_{w}(0)$}\\
\bzero, & \text{otherwise},
\end{cases}
\end{align}
where $dT$ is the sampling time and  $E_{z}(0)$, $E_{w}(0)$ are initially stored energy.

\subsubsection{Manipulator control (compliance control)}

The manipulator control input to realize compliance behavior of the end-effector is given by
\begin{align}
\label{eq:control_tau_m}
\btau_m &= \bJ_{eq}^T \bT^T \left( \bK_p(\bx_e^{des} - \bx_e) - \bK_d \dot{\bx}_e  \right) + \bg_m,
\end{align}
where $\bK_p, \bK_d > \bzero$ are stiffness and damping gains respectively, and $\bT=diag\{{}^{f}\bR_{e}, {}^{e}\bQ_f^{-1}\}$. Similar to the UAV control, the UAV variables do not appear in the manipulator control except for the gravity compensation.

To show the intrinsic passivity of the compliance controller, we begin with the following lemma.
\begin{lem}
	\label{lem:passive_pre_post_mult}
	Consider an arbitrary system with passive I/O pair $(\bm{u}, \bm{y} )$. Pre-multiplication of the input $\bm{u}$ by a matrix $\bm{A}$ and post-multiplication of the output $\bm{y}$ by $\bm{A}^T$ preserve the passivity.  Namely,  the new I/O pair $(\tilde{\bm{u}}, \tilde{\bm{y}} )$ is passive, where $\bm{u}=\bm{A}\tilde{\bm{u}}$ and $\tilde{\bm{y}}=\bm{A}^T \by$. 
\end{lem}
\begin{proof}
	By assumption, $\int \bu^T \by> 0$. The passivity of the new I/O pair is trivial because $\int \tilde{\bu}^T \tilde{\by} = \int \tilde{\bu}^T \bm{A}^T \by= \int (\bm{A}\tilde{\bu})^T\by= \int \bu^T\by > 0$.
\end{proof}


The following lemma shows that $\dot{\bx}_e$ can be obtained by a coordinate transformation of $\bxi_e$.

\begin{lem}
	\label{lem:dot_x_e}
	$\dot{\bx}_e$ can be expressed as
	\begin{align}
	\dot{\bx}_{e} 
	\label{eq:xdot_e_2}
	=& \bT [-\bJ_{ef} \;\; \bI]\bxi_e.
	\end{align}
\end{lem}
\begin{proof}	
	Using
	\begin{align}
	[\bzero \;\; \bI]\bxi_e = \bV_e = &  \underbrace{ \bJ_{ef}  \bV_f }_{={}^{e}\bV_{f/g}} + \underbrace{\bJ_{eq}\dot{\bq}_m}_{={}^{e}\bV_{e/f}}\\
	=& [ \bJ_{ef} \;\; \bzero] \bxi_e + {}^{e}\bV_{e/f},
	\end{align}
	the following relation holds:
	\begin{align}
	\label{eq:proof_xdot_e_1}
	{}^{e}\bV_{e/f} = [ -\bJ_{ef} \;\; \bI]\bxi_e.
	\end{align}
	Furthermore, noting that
	\begin{align}
	\label{eq:proof_xdot_e_2}
	{}^{e}\bV_{e/f} = 
	\left(
	\begin{array}{c}
	{}^{e}\bR_{f}{}^{f}\dot{\br}_{e/f} \\
	{}^{e}\bw_{e/f}
	\end{array}
	\right)
	=
	\underbrace{\left[
		\begin{array}{cc}
		{}^{e}\bR_{f} & \bzero\\
		\bzero & {}^{e}\bQ_{f}
		\end{array}
		\right]}_{=\bT^{-1}}
	\dot{\bx}_e,
	\end{align}
	we arrive at (\ref{eq:xdot_e_2}) by combining (\ref{eq:proof_xdot_e_1})-(\ref{eq:proof_xdot_e_2}).
\end{proof}

Finally, the following theorem shows the intrinsic passivity of compliance controller (\ref{eq:control_tau_m}) using Lemma \ref{lem:passive_pre_post_mult} and \ref{lem:dot_x_e}.

\begin{thm}
	\label{thm:passivity_controller}
	The compliance controller (\ref{eq:control_tau_m}) satisfies passivity of the I/O pair $(-\bxi_e, {}^{e}\bEta_{m})$.
\end{thm}
\begin{proof}
	We begin with the fact that set-point velocity PI control (which is equivalent to the set-point position PD control) is passive. Using Lemma \ref{lem:dot_x_e}, the block diagram from $-\bxi_e$ to ${}^{e}\bEta_{m}$ can be described by Fig. \ref{fig:passivity_compliance_control}. Hence the I/O pair $(-\bxi_e, {}^{e}\bEta_{m})$ is passive by Lemma \ref{lem:passive_pre_post_mult}.
\end{proof}

Please note that, to preserve passivity, $\bT[-\bJ_{ef} \; \bI]$ has to be post-multiplied to $\bxi_{e}$, and this leads to $\bx_{e}$ as a control variable.  The following section presents stability and passivity analysis of the controlled UAV-M system.

\subsection{Stability of the controlled UAV-M system}

The following theorem states stability during the free-flight.

\begin{thm}[Asymptotic stability for free flight]
	\label{thm:stability}
	Assume that ${}^{e}\bF_{\text{ext}}=0$. If the control gains $D_z$, $\bD_w$, $\bD_\phi$, $\bK_p$, and $\bK_d$ are chosen sufficiently large, then the closed-loop dynamics resulting from (\ref{eq:control_f_th})-(\ref{eq:control_tau_uav}), (\ref{eq:control_tau_m}) is asymptotically stable to ${}^{g}\br_f={}^{g}\br_f^{des}$, ${}^{g}\bphi_f={}^{g}\bphi_f^{des}=[0 \; 0\; \gamma_f^{des}]^T$, and  $\bx_e=\bx_e^{des}$ with zero derivatives.
\end{thm}
\begin{proof}
	See Appendix.
\end{proof}

\begin{rem}
	\label{rem:difference}
	In \cite{kim2018stabilizing}, stability is analyzed based on (i) perfect $\bq_m$ regulation and (ii) perfect $\dot{\bw}_f^{ref}$ tracking. To achieved these, feedback linearization was used in \cite{kim2018stabilizing}. However, feedback linearizing action includes coupling between UAV and manipulator, and consequently, the resulting control law becomes highly model dependent. In principle, model-based controllers will probably outperform the model-free ones, but at the cost of increased implementation complexity. In contrast, (\ref{eq:control_f_th})-(\ref{eq:control_tau_uav}) and (\ref{eq:control_tau_m}) do not include feedback linearizing action. In the stability proof, this paper uses two-time scale analysis. This is a reasonable choice because the un-actuated $x$-,$y$-translational dynamics is slower than the actuated dynamics because the former is indirectly controlled by the rotational dynamics of UAV. \QEDB
\end{rem}

\begin{figure}[]
	\centering
	{\includegraphics[scale=0.39]{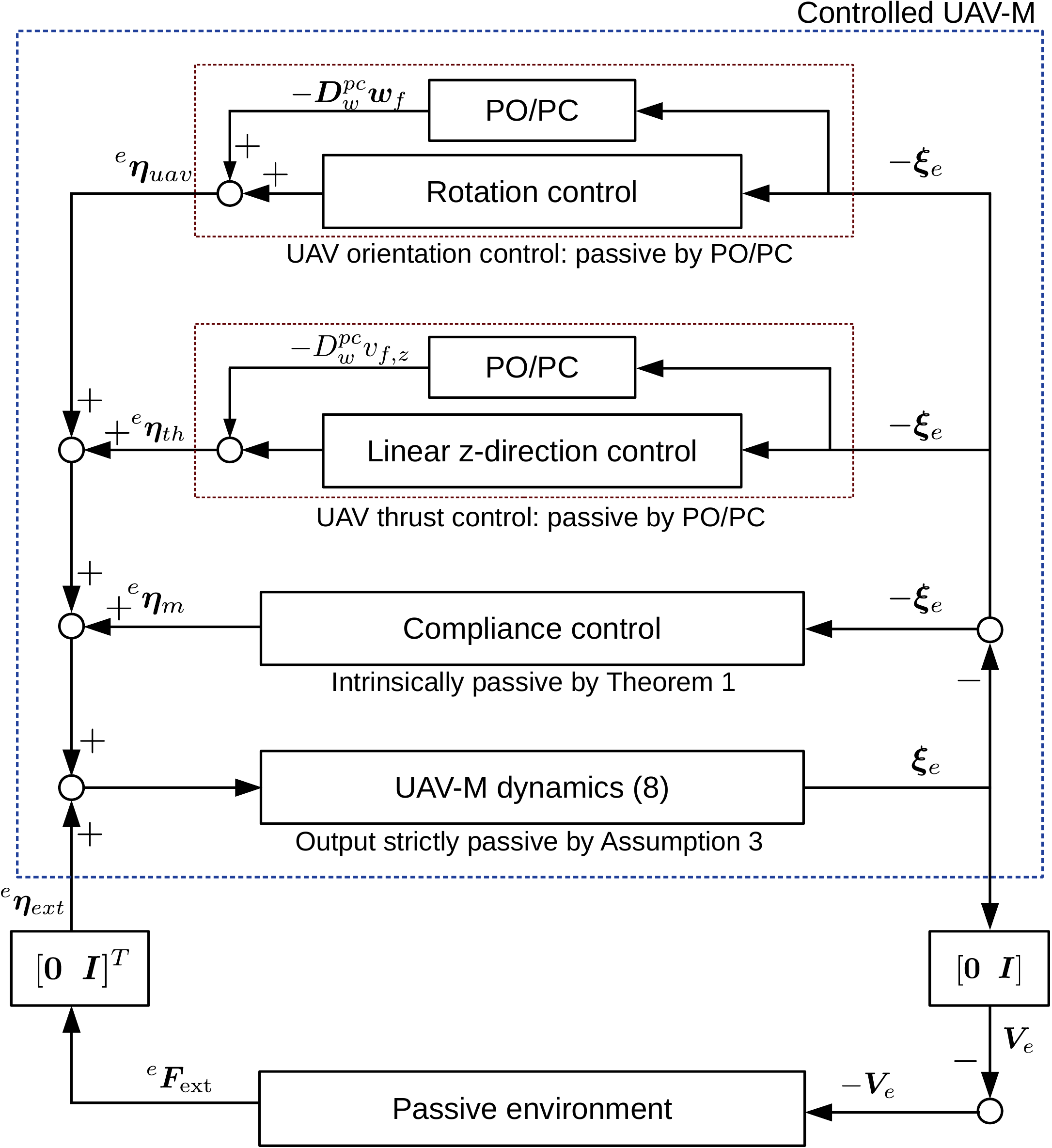}} \\
	\caption{The resulting control structure can be interpreted as feedback interconnections of output strictly and strictly passive subsystems. As a result, the overall closed-loop dynamics (including passive environment) is asymptotically stable.}
	\label{fig:overall_control_str}
\end{figure}

Summarizing the discussion in Section \ref{sec:control_design}, the closed-loop dynamics can be described by Fig. \ref{fig:overall_control_str}. Using skew-symmetric property of $\dot{\bLambda} - 2\bGamma$,  the controlled UAV-M (blue dashed box in Fig. \ref{fig:overall_control_str}) can be represented as feedback interconnections of passive subsystems. The following theorem states the stable environmental interaction.



\begin{thm}[Stable environmental interaction]
	\label{thm:passivity}
	Assume that the environment can be modeled as a mass-spring-damper system so that Assumption \ref{ass:passive_env} is valid. The overall closed-loop dynamics including the environment is asymptotically stable to a certain equilibrium point.
\end{thm}
\begin{proof}
	The I/O pair of the controlled UAV-M $({}^{e}\bEta_{\text{ext}}, \bxi_e)$ is output strictly passive due to Assumption \ref{ass:strictly_passive_UAV_M}. By small extension of Lemma \ref{lem:passive_pre_post_mult}, it can be easily shown that pre-/post-multiplication of the I/O by $\bm{A} = [\bzero \; \bI]^T$ preserves the output strict passivity. Therefore, the I/O pair $({}^{e}\bF_{\text{ext}}, \bV_{e})$ is output strictly passive. Noting that Theorem \ref{thm:stability} implies zero-state observability of the controlled UAV-M, and that the I/O pair $(-\bV_{e}, {}^{e}\bF_{\text{ext}})$ of the environment is strictly passive, asymptotic stability of the entire closed-loop shown in Fig. \ref{fig:overall_control_str} can be concluded by applying Theorem 6.3 in \cite{khalil2002nonlinear}.
%
\end{proof}

This theorem requires to model the environment using spring and damper elements. Because this argument must hold even with negligibly small damping value, the environment will be modeled as a pure spring in the simulation validations as an extreme case.

%
%
%

\begin{figure*}[]
	\centering
	{\includegraphics[scale=0.62]{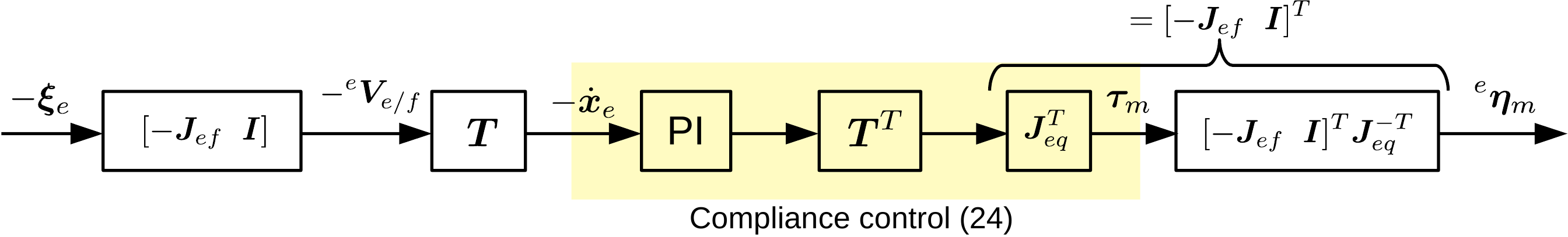}} \\
	\caption{Passivity preserving coordinate transformation of the compliance controller.}
	\label{fig:passivity_compliance_control}
\end{figure*}

%
%

\section{Simulation Validation}
\label{sec:validation}

To validate the proposed approach, the UAV-M in Fig. \ref{fig:heli_manipulator} was simulated using rigid body dynamics (\ref{eq:modified_dyn}). The mass and inertia of UAV were $37.6$  $\mathrm{kg}$ and  $diag\{1.46 \; 0.36\; 1.46\}$ $\mathrm{kg\cdot m^2}$. Due to Assumption \ref{ass:redundancy}, the 7th joint of the DLR light weight robot (LWR) manipulator was fixed. For every simulation, initial manipulator position was $\bx_{e}(0)=[0.3 \;\; 0 \;\; -0.75 \;\; 0 \;\; 0 \;\; 0]^T$ (recall that $\bx_e$ is the position/orientation of the end-effector from the UAV fuselage). For compliance controller, the stiffness and damping gains of translation were $diag\{100, 100, 100\}\mathrm{N/m}$, $diag\{10, 10, 10\}\mathrm{N\cdot s/m}$, and those of rotation were $diag\{100, 100, 100\}\mathrm{Nm/rad}$, $diag\{10, 10, 10\}\mathrm{Nm\cdot s/rad}$, respectively.

\subsection{Asymptotic stability during free flight}

To validate stability during free flight, the following task was performed.
\begin{itemize}
	\item The desired UAV position: ${}^{g}\br_{f}^{des}=\bzero \rightarrow [10 \; 10\; 10]^T\mathrm{m}$.
	\item The desired UAV orientation: ${}^{g}\bphi_{f}^{des}=\bzero\mathrm{rad}$.
	\item The desired end-effector position: At $t=15$ $\mathrm{s}$, $\bx_{e}^{des}=[0.3 \; 0  \; - 0.75 \; 0 \; 0 \; 0]^T \rightarrow [-0.5 \; 0 \; -0.75 \; 0 \; 0 \; 0 ]^T$.
\end{itemize}
Here, $\rightarrow$ means the step command.  

The results are shown in Fig. \ref{fig:simul_res_stability}. As the UAV position error occurred at the beginning (because of the step command), the reference acceleration $\dot{\bw}_f^{ref}$ defined in  (\ref{eq:bw_f_dot_ref}) was excited. Consequently, UAV roll and pitch angles were perturbed to reduce the UAV position error, and asymptotic stability could be achieved. Notice that the manipulator was also perturbed due to the UAV motion, because of the dynamic coupling between UAV and manipulator. Recall that the controller (\ref{eq:control_f_th})-(\ref{eq:control_tau_uav}) is (almost) model-free and does not cancel out the dynamic coupling. Conversely, at $t=15\mathrm{s}$, UAV dynamics was perturbed due to the manipulator's motion (see the magnified view in the first two rows), and was stabilized in a short instant.

\subsection{Stable interaction with passive environment}

To validate stable environmental interaction, the following two tasks were performed.

\subsubsection{Task 1}

In this task, physical wall located at $0.4$ $\mathrm{m}$ in global $x$-direction was simulated using $1000$ $\mathrm{N/m}$ spring, as shown in Fig. \ref{fig:simul_env}a. Desired positions were set as follows.
\begin{itemize}
	\item The UAV position: ${}^{g}\br_{f}^{des}=\bzero\mathrm{m}$ (hovering).
	\item The UAV orientation: ${}^{g}\bphi_{f}^{des}=\bzero\mathrm{rad}$.
	\item The end-effector position: $\bx_{e}^{des}=[0.3 \;\; 0$$ \;\; - 0.75 \;\; 0 \;\; 0 \;\; 0]^T \rightarrow [0.6 \;\; 0  \;\; - 0.75 \;\; 0 \;\; 0 \;\; 0]^T$.
\end{itemize}

The simulation results are shown in Fig. \ref{fig:simul_res_wall}. After contact occurred (around $t=0.5$ $\mathrm{s}$), the external force pushed the UAV away, and consequently, external force became zero because the contact was lost. Due to the interaction, the UAV position was disturbed by about $0.2\mathrm{m}$. The contact occurred again as the UAV recovered its desired position; see the third and fourth rows of Fig. \ref{fig:simul_res_wall}. During the contact, UAV-M converged to a certain equilibrium point at which it can balance the external force; $\alpha_f$ converged to a certain value  (recall Fig. \ref{fig:problem_state}b). Notice that the UAV converged to this equilibrium point automatically by virtue of passivity, without calculating the orientation that balances the interaction force. Note also that oscillation occurred on the end-effector (and hence on the interaction force as well) because the environment was modeled as a pure spring, but the was dissipated eventually. In this simulation, PC for UAV dynamics was not activated because the $E_{obs,w}$ and $E_{obs,z}$ were always positive.

\subsubsection{Task 2}

In this task, we consider the forces acting on the manipulator along every direction. Therefore, the environment is modeled as springs in every translational direction with stiffness of $1000$ $\mathrm{N/m}$, as shown in Fig. \ref{fig:simul_env}b. Note that this task considers only environmental interaction, whereas the previous task includes both free-flight phase and interaction phase. Desired positions were set as follows.


\begin{itemize}
	\item The UAV position: ${}^{g}\br_{f}^{des}=\bzero\mathrm{m}$ (hovering).
	\item The UAV orientation: ${}^{g}\bphi_{f}^{des}=\bzero\mathrm{rad}$.
	\item The end-effector position: $\bx_{e}^{des}=[0.3 \; 0$$\; - 0.75 \; 0 \;\; 0 \;\; 0]^T \rightarrow [0.6 \;\; 0.3  \;\; - 1.0 \; 0 \;\; 0 \;\; 0]^T$.
\end{itemize}

The simulation results are shown in Fig. \ref{fig:simul_res_constrained}. Similar to task 1, the UAV orientation converged to a certain equilibrium point that balances the interaction force, while the UAV position converged to the desired position. However, in this task, mainly because of the integral action, the passivity condition was broken in body linear $z$-direction, and the PC was activated. By virtue of the PC, the energy was maintained to be positive (fourth row of Fig. \ref{fig:simul_res_constrained}). In addition, the energy of I/O port $(-\bxi_e, {}^{e}\bEta_m)$ was observed to validate the passivity of the compliance controller. Passivity of this I/O port was maintained during the interaction because the energy was always less than the initially stored energy (see the fifth row of Fig. \ref{fig:simul_res_constrained}; 12.125$\mathrm{J}$ is the initially stored energy\footnote{
	Initial spring displacement is $[0.3 \; 0.3\; -0.25]^T\mathrm{m}$ and the stiffness was $diag\{100,100,100\}\mathrm{N/m}$.}). Therefore, we can conclude that the controlled UAV-M was passive because every sub-block in Fig. \ref{fig:overall_control_str} was passive. As a result, stable interaction with the environment could be achieved.

\begin{figure}
	\centering
	{\includegraphics[scale=0.54]{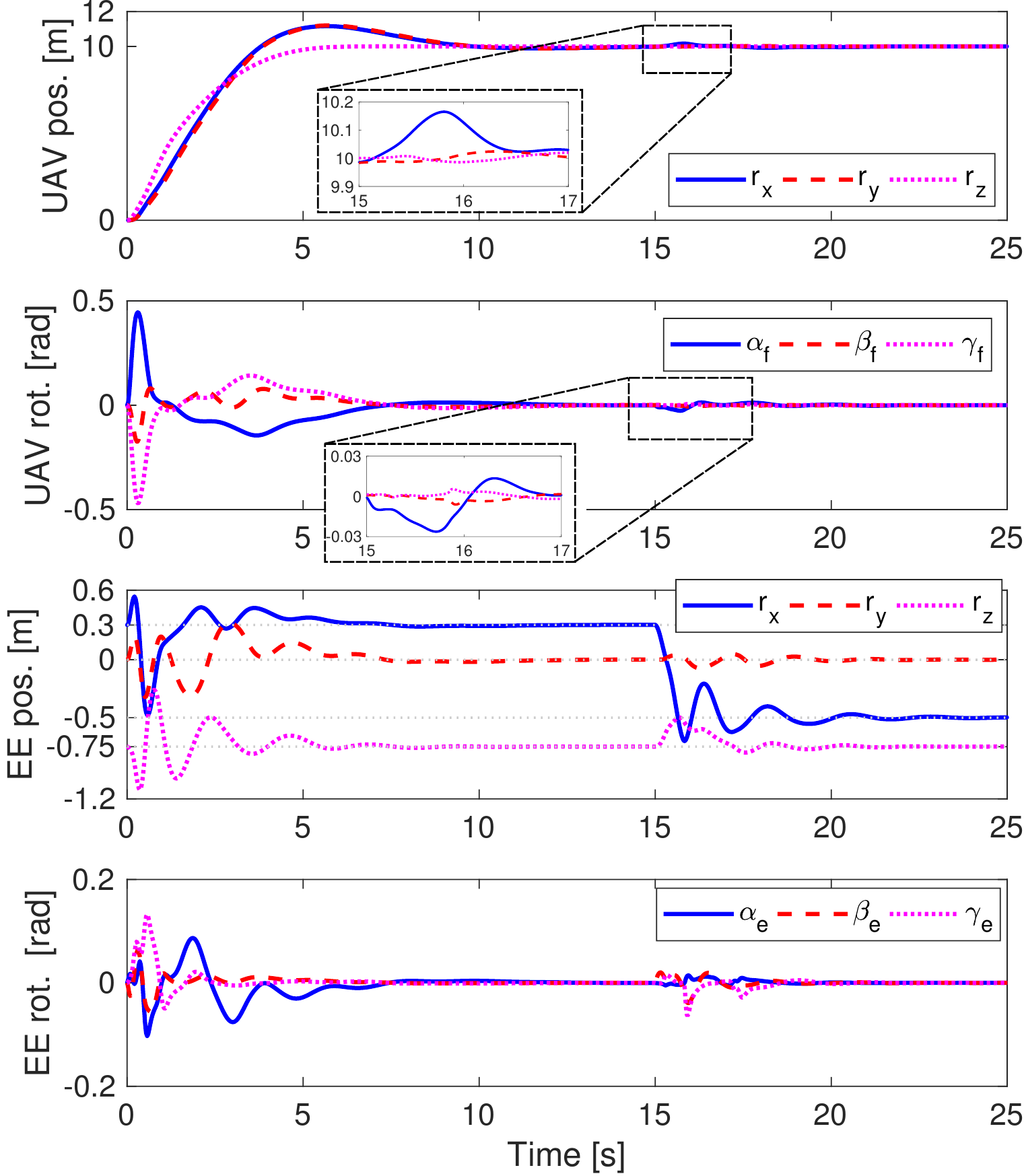}}
	\caption{Simulation validation for asymptotic stability during free flight. First and second rows: UAV position and orientation. Third and fourth rows: end-effector position and orientation ($\bx_e$).}
	\label{fig:simul_res_stability}
\end{figure}

\section{Conclusion}
\label{sec:conclusion}

This paper presents a passive compliance control for UAV-M to ensure stable interaction with passive environments. The key finding of this study is that the compliance controller satisfies passivity if the position of the end-effector is represented in the UAV fuselage frame. In addition to the passive manipulator controller, PO/PC technique is applied to ensure passivity of UAV controller. As a result, the controlled UAV-M can interact with the passive environment stably. Simulation studies validate stability of free flight and stable environmental interaction.

\section*{Appendix}
Since the control laws (\ref{eq:control_f_th})-(\ref{eq:control_tau_uav}) and (\ref{eq:control_tau_m}) are the extension of our previous work \cite{kim2018stabilizing}, this section presents only the sketch of the proof. To begin with, let us express a new coordinates $\bbxi$ which has $\dot{\bx}_e$ instead of $\bV_e$.
\begin{align}
\bar{\bxi} = 
\underbrace{
	\left[
	\begin{array}{cc}
	\bI & \bzero \\ 
	\bzero & 	\bT \bJ_{eq}
	\end{array}
	\right]
}_{=\bar{\bJ}}
\bxi_q.
\end{align}
Recall $\bT$ introduced in (\ref{eq:proof_xdot_e_2}). Using this coordinates, the UAV-M dynamics is
\begin{align}
\label{eq:uav_m_dynamics_for_stability}
\bbM \dot{\bbxi} + \bbC \bbxi + \bar{\bJ}^{-T} \bg = \bar{\bJ}^{-T} \btau_b,
\end{align}
with properly defined $\bbM$ and $\bbC$.

In the following, we apply two-time scale (also known as singular perturbation) analysis in which the linear $x$-, $y$-directional dynamics becomes slow dynamics and rotational dynamics become fast dynamics. This analysis is reasonable because the linear $x$-, $y$-directional motions are the consequences of orientation of the UAV.

For analysis purpose, let $D_z=1/\epsilon$, $\bD_w=\frac{1}{\epsilon}\bI$, $\bK_d=\frac{1}{\epsilon}\bI$, and $\bK_p=\frac{1}{\epsilon^2}\bI$. Also, let us define the fast time scale $\sigma$ by
\begin{align}
\sigma = \frac{1}{\epsilon}t,
\end{align}
and the derivative of $(\cdot)$ with respect to $\sigma$ is defined as
\begin{align}
(\cdot)' = \frac{d}{d \sigma}(\cdot)  = \frac{d}{dt/\epsilon}(\cdot) = \epsilon \frac{d}{dt}(\cdot).
\end{align}

Note that $v_{f,x}$ and $v_{f,y}$ are the frozen variables in the fast time scale, because first two rows of (\ref{eq:uav_m_dynamics_for_stability}) which represent linear $x$- and $y$-directional dynamics can be expressed as
\begin{align}
\frac{d}{d \sigma}
\left(
\begin{array}{c}
v_{f,x} \\
v_{f,y}
\end{array}
\right)
=0.
\end{align}
Then, as $\epsilon \rightarrow 0$, (\ref{eq:uav_m_dynamics_for_stability}) can be written as
\begin{align}
\label{eq:boundary_1}
\bbM_r 
\left(
\begin{array}{c}
v'_{f,z}\\
\bw'_f \\
\bx''_e
\end{array}
\right)
+
\left(
\begin{array}{c}
\tilde{v}_{f,z}\\
\bw_f - \bw_f^{ref} \\
\bx'_e + (\bx_e-\bx_e^{des})
\end{array}
\right)
=\bzero.
\end{align}
Here, $\bbM_r>\bzero$ denotes the last $(4+n)\times(4+n)$ block matrix (the subscript `$r$' stands for reduced). Therefore, it is trivial that the states converge to $\tilde{v}_{f,z}=0$, $\bw_{f}=\bw_{f}^{ref}$, and $\bx_e=\bx_e^{des}$ in the fast time scale $\sigma$.

\begin{figure}
	\centering
	\subfigure[]
	{\includegraphics[scale=0.49]{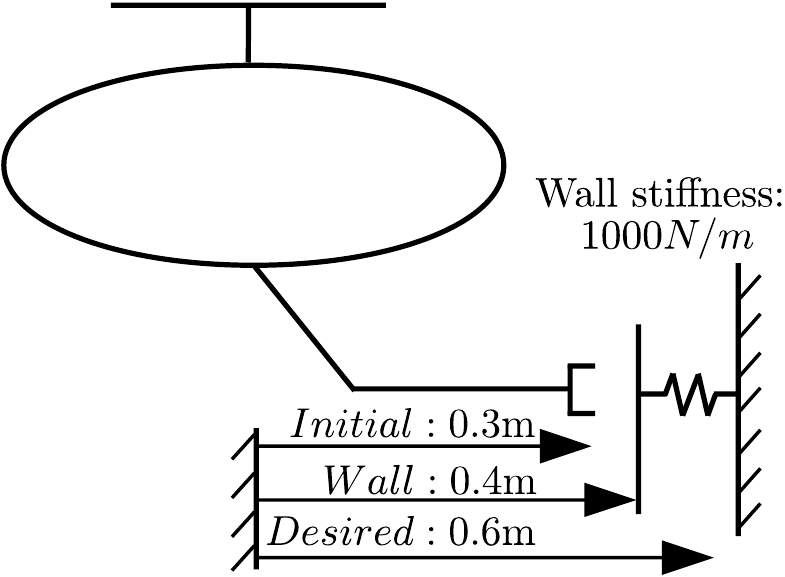}}
	\centering
	\subfigure[]
	{\includegraphics[scale=0.49]{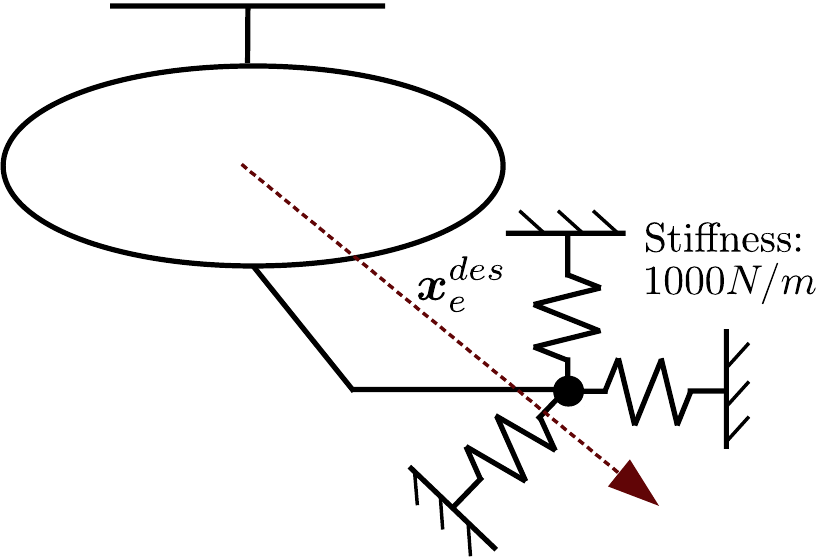}}
	\caption{Two tasks for validation of stable interaction. (a) Task 1: Physical wall is simulated with wall stiffness 1000 $\mathrm{N/m}$. The end-effector was commanded to penetrate the wall. (b) Task 2: End-effector is constrained by springs in linear $x$-,$y$-,$z$-directions, and commanded to move $[0.3 \; 0.3 \; 0.25]\mathrm{m}$ forward.}
	\label{fig:simul_env}
\end{figure}

Rest of the analysis can be performed by following the same procedure introduced in \cite{kim2018stabilizing}. As stated in Remark \ref{rem:difference}, the only difference is that  \cite{kim2018stabilizing} achieved $\bw_{f}=\bw_{f}^{ref}$ and $\bx_e=\bx_e^{des}$ by feedback linearization, but this paper achieved them by two-time scale analysis.

\begin{figure}
	\centering
	{\includegraphics[scale=0.54]{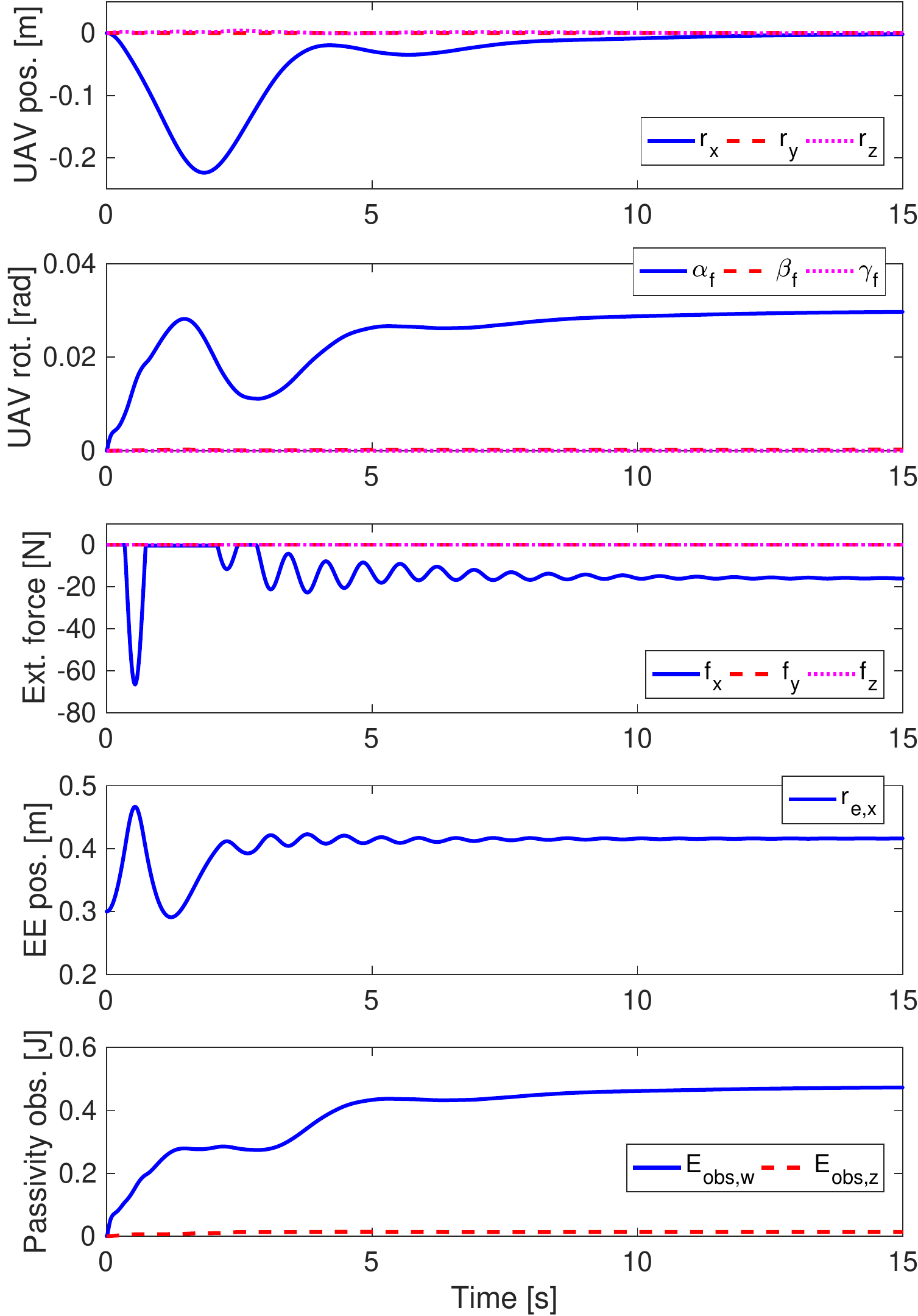}}
	\caption{Simulation result to validate stable environmental interaction for task 1. First and second rows: UAV position and orientation. UAV position converged to the desired even with the environmental interaction, and the orientation converged to a certain value to balance the interaction force. Third row: interaction force due to the contact with the environment (the wall) which is modeled as a pure spring. Fourth row: end-effector position in the global frame (i.e., $r_{e,x}$). Fifth row: Passivity observers. In this simulation, PC was not activated because POs were positive. }
	\label{fig:simul_res_wall}
\end{figure}

\begin{figure}
	\centering
	{\includegraphics[scale=0.54]{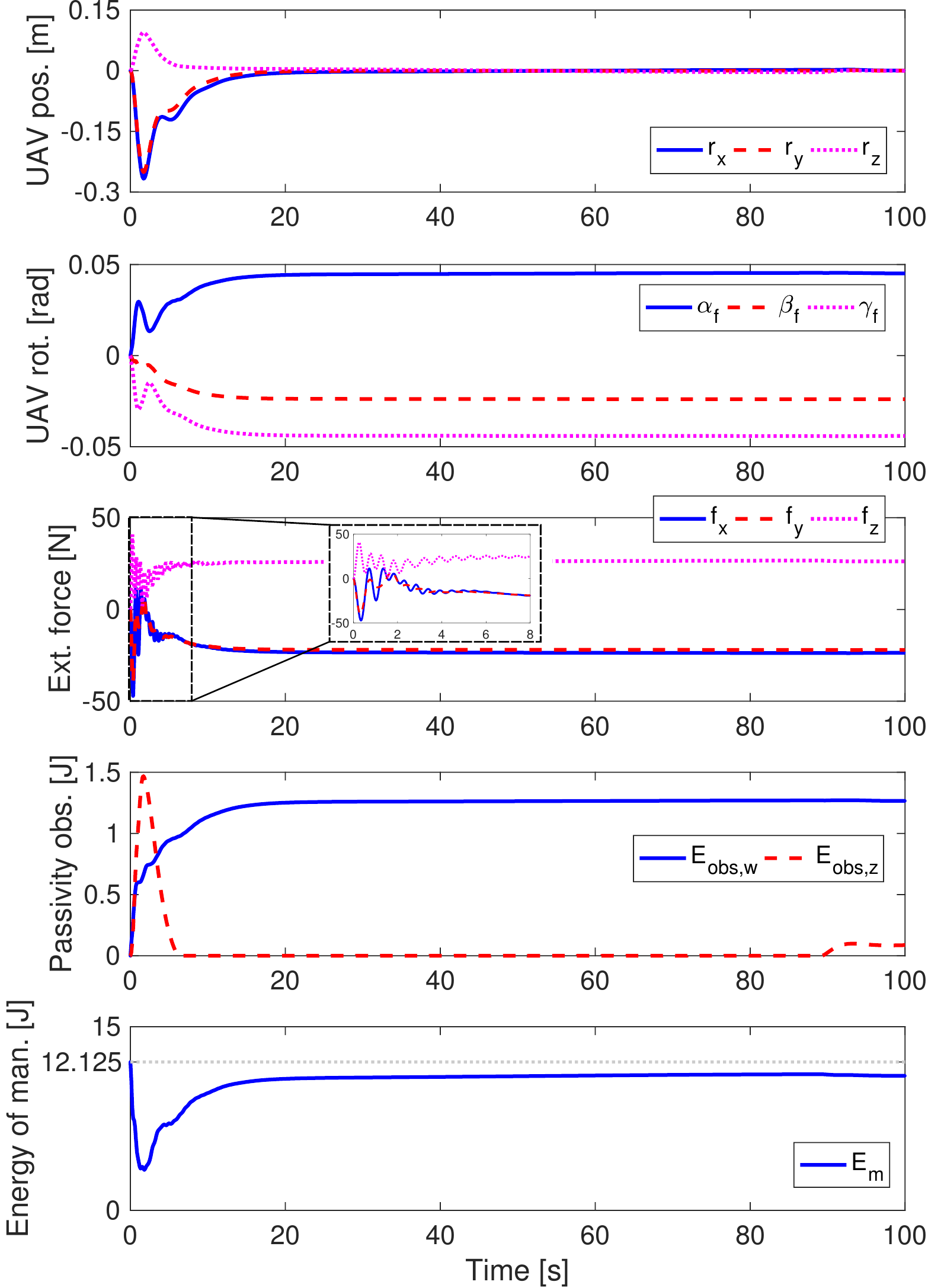}}
	\caption{Simulation result to validate stable environmental interaction for task 2. First and second rows: UAV position and orientation. Similar to previous, the orientation converged to a certain value to balance the interaction force, while position was converged to the desired. Third row: interaction force. Fourth row: passivity observers. PC for body $z$-direction was activated to keep $E_{obs,z}$ positive. Fifth row: Energy of I/O port $(-\bxi_e, {}^{e}\bEta_m)$; 12.125 $\mathrm{J}$ was the initially stored energy. Fourth and fifth row confirm passivity of the controlled UAV-M because every subsystem is passive.}
	\label{fig:simul_res_constrained}
\end{figure}

\bibliographystyle{IEEEtran}
\bibliography{IEEEabrv,[bib]IROS2018_aerial}

\end{document}